%% file: cvpr2017laplacian_arxiv.tex
\documentclass[10pt,twocolumn,letterpaper]{article}

\usepackage{cvpr}
\usepackage{booktabs}
\usepackage{times}
\usepackage{epsfig}
\usepackage{graphicx}
\usepackage{amsmath}
\usepackage{amssymb}
\usepackage{amsthm}
\usepackage{array}
\usepackage{mathtools}
\usepackage{algorithm}
\usepackage{algorithmic}
\usepackage{subfig}
\usepackage{varwidth}
\usepackage{mdframed}
\usepackage[export]{adjustbox} 
 
\graphicspath{{./figs/}}
\include{Macros}

\newcommand{\tr}[1]{\text{tr}\left( #1\right)}
\newcolumntype{L}[1]{>{\raggedright\let\newline\\\arraybackslash\hspace{0pt}}m{#1}}
\newcolumntype{C}[1]{>{\centering\let\newline\\\arraybackslash\hspace{0pt}}m{#1}}
\newcolumntype{R}[1]{>{\raggedleft\let\newline\\\arraybackslash\hspace{0pt}}m{#1}}
\newcommand{\ra}[1]{\renewcommand{\arraystretch}{#1}}

\newcommand{\OO}[1]{{\rm O}(#1)}    
\renewcommand{\SO}[1]{{\rm SO}(#1)}    
\renewcommand{\AA}{W}

\usepackage[pagebackref=true,breaklinks=true,letterpaper=true,colorlinks,bookmarks=false]{hyperref}

\cvprfinalcopy 

\def\cvprPaperID{984} 

\ifcvprfinal\fi

\begin{document}

\clearpage
\title{Rotation Averaging and Strong Duality}

\author{Anders Eriksson$^1$,\hspace{2mm} Carl Olsson$^{2,3}$,\hspace{2mm} Fredrik Kahl$^{2,3}$ and\hspace{1mm} Tat-Jun Chin$^4$ \vspace{0.5cm} \\
$^1$School of Electrical Engineering and Computer Science, Queensland University of Technology\\
$^2$Department of Electrical Engineering,
Chalmers University of Technology\\
$^3$Centre for Mathematical Sciences, Lund University\\
$^4$School of Computer Science, The University of Adelaide
}

\maketitle

\begin{abstract}
In this paper we explore the role of duality principles within the problem of rotation averaging, a fundamental task in a wide range of computer vision applications. In its conventional form, rotation averaging is stated as a minimization over multiple rotation constraints. As these constraints are non-convex, this problem is generally considered challenging to solve globally. We show how to circumvent this difficulty through the use of Lagrangian duality. While such an approach is well-known it is normally not guaranteed to provide a tight relaxation. Based on spectral graph theory, we analytically prove that in many cases there is no duality gap unless the noise levels are severe. This allows us to obtain certifiably global solutions to a class of important non-convex problems in polynomial time. 

We also propose an efficient, scalable algorithm that out-performs general purpose numerical solvers and is able to handle the large problem instances commonly occurring in structure from motion settings. The potential of this proposed method is demonstrated on a number of different problems, consisting of both synthetic and real-world data. \end{abstract}


\vspace{-7mm}
\section{Introduction}


Rotation averaging appears as a subproblem in many important applications in computer vision, robotics, sensor networks and related areas. Given a number of relative rotation estimates between pairs of poses, the goal is to compute absolute camera orientations with respect to some common coordinate system. In computer vision, for instance, non-sequential structure from motion systems such as \cite{martinec-pajdla-cvpr-2007,enqvist2011non,moulon-etal-iccv-2013} rely on rotation averaging to initialize bundle adjustment.
The overall idea is to consider as much data as possible in each step to avoid suboptimal reconstructions. In the context of rotation averaging this amounts to using as many camera pairs as possible.

The problem can be thought of as inference on the camera graph. An edge $(i,j)$ in this undirected graph represents a relative rotation measurement $\Rij$ and the objective is to find the absolute orientation $R_i$ for each vertex $i$ such that $R_i \Rij =R_j$ holds (approximately in the presence of noise) for all edges.
The problem is generally considered difficult due to the need to enforce non-convex rotation constraints. Indeed, both $L_1$ and $L_2$ formulations of rotation averaging can have local minima, see Fig.~\ref{fig:minima}. Wilson {\em et al.} \cite{Wilson2016} studied local convexity of the problem and showed that instances with large loosely connected graphs are hard to solve with local, iterative optimization methods.

In contrast, our focus is on global optimality. In this paper we show that convex relaxation methods can in fact overcome the difficulties with local minima in rotation averaging. We utilize Lagrangian duality to handle the quadratic non-convex rotation constraints. While such an approach is normally not guaranteed to provide a tight relaxation we give analytical error bounds that guarantee there will be no duality gap. For instance, it is sufficient that each angular residual is less than $42.9^\circ$ to ensure optimality for complete camera graphs.
Additionally, we develop a scalable and efficient algorithm, based on block coordinate descent, that outperforms standard semidefinite program (SDP) solvers for this problem. 

\begin{figure}[t]
\begin{center}
\vspace{-3mm}
   \setlength\tabcolsep{0pt} 
   \begin{tabular}{ccc} \hspace{-0.3cm}
   \includegraphics[width=0.38\linewidth]{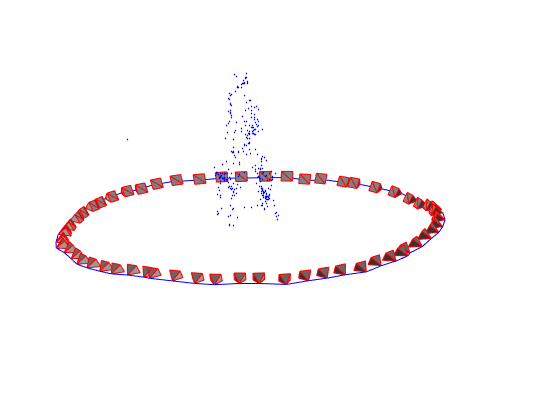} \hspace{-0.4cm} &
   \includegraphics[width=0.38\linewidth]{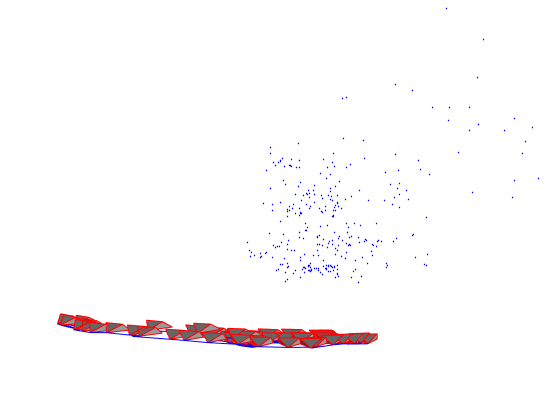} \hspace{-0.7cm} &
   \includegraphics[width=0.38\linewidth]{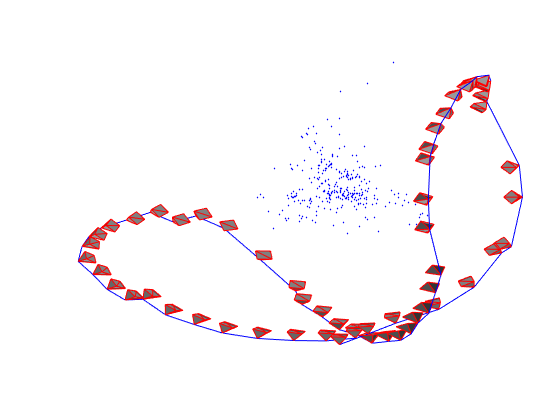} \vspace{-0.6cm}
  \end{tabular} 
   \setlength\tabcolsep{6pt} 
\end{center}
   \caption{In many structure from motion pipelines, camera orientations are estimated with rotation averaging followed by recovery of camera centres (red) and 3D structure (blue). Here are three solutions corresponding to different local minima of the same rotation averaging problem. }
\label{fig:minima}
\end{figure}

\paragraph{Related work. } Rotation averaging has been under intense study in recent years, see \cite{hartley2013,kahl2008multiple,martinec-pajdla-cvpr-2007,arrigoni2014robust,tron2012intrinsic,corleone-etal-icra-2015}. 
Despite progress in practical algorithms, they largely come without guarantees. One of the earliest averaging methods was due to Govindu \cite{govindu-cvpr-2001}, who
showed that when representing the rotations with quaternions the problem can be viewed as a linear homogeneous least squares problem. There is however a sign ambiguity in the quaternion representation that has to be resolved before the formulation can be applied. 
It was observed by Fredriksson and Olsson in \cite{fredriksson2012simultaneous} that since both the objective and the constraints are quadratic, the Lagrange dual can be computed in closed form. The resulting SDP was experimentally shown to have no duality gap for moderate noise levels.

A more straightforward rotation representation is $3\times 3$ matrices. Martinec and Pajdla \cite{martinec-pajdla-cvpr-2007} approximately solve the problem by ignoring the orthogonality and determinant constraints.
A similar relaxation was derived by Arie-Nachimson {\em et al.} in \cite{arie2012global}. In addition, an SDP formulation was presented which is equivalent to the one we address here, but with no performance guarantees.
The tightness of SDP relaxations for 2D rotation averaging is studied in \cite{zhong-baumal-arxiv-2017}.

A number of robust approaches have been developed to handle outlier measurements.
A sampling scheme over spanning trees of the camera graph is developed by Govindu in \cite{govindu-eccv-2006}. Enqvist {\em et al.} \cite{enqvist2011non} also start from a  spanning tree and add relative rotations that are consistent with the solution.
In \cite{hartley-etal-cvpr-2011} the Weiszfeld algorithm is applied to single rotation averaging with the $L_1$ norm. In \cite{hartley-etal-2010} convexity properties of the single rotation averaging problem are given. 
To our knowledge these results do not generalize to the case of multiple rotations. In \cite{chatterjee-govindu-iccv-2013} a robust formulation is solved using IRLS and in \cite{boumal2014crbsynch} Cram\'er-Rao lower bounds are computed for maximum likelihood estimators, but neither with any optimality guarantees.

A closely related problem is that of pose graph estimation, where camera orientations and positions are jointly optimized. In this context Lagrangian duality has been applied \cite{carlone-etal-tr-2016,carlone-dellaert-icra-2015}.
In \cite{tron-vidal-icra-2014} a consensus algorithm that allows for efficient distributed computations is presented. A fast verification technique for pose graph estimation was given in \cite{briales2016fast}. In a recent paper \cite{rosen-etal-arxiv-2017} an SDP relaxation for pose graph estimation with performance guarantees is analyzed. It is shown that there is a noise level $\beta$ for which the relaxation is guaranteed to provide the optimal solution. However, the result only shows the existence of $\beta$. Its value which is dependent on the problem instance is not computed. In contrast our result for rotation averaging gives explicit noise bounds. 

The main contributions of this paper are: \vspace{-0.5mm}
\begin{itemize}
\item We apply Lagrangian duality to the rotation averaging problem with the chordal error distance and study the properties of the obtained relaxations.\vspace{-0.5mm}
\item We develop strong theoretical bounds on the noise level that guarantee exact global recovery based on spectral graph theory.\vspace{-0.5mm}
\item We develop a conceptually simple and scalable algorithm which is able to handle large problem instances occurring in structure from motion problems.\vspace{-0.5mm}
\item We present experimental results that confirm our theoretical findings.
\end{itemize}

\subsection{Notation and Conventions}
Let $G=(V,E)$ denote an undirected graph with vertex set $V$ and edge set $E$ and let $n = |V|$. The adjacency matrix $A$ is by definition the $n\times n$ matrix with elements
\begin{align}
a_{ij}=
\left\{ 
\Arraj{
0 & (i,j)\notin E \\
1 & (i,j)\in E 
} \mbox{ for }i,j = 1,\ldots,n.
\right.
\end{align}
The degree $d_i$ is the number of edges that touch vertex $i$, and the degree matrix $D$ is the diagonal matrix $D = \text{diag}\left(d_1,\ldots,d_n\right)$. The Laplacian $L_G$ of $G$ is defined by
\begin{align}
L_G = D-A.
\end{align}
It is well-known that $L_G$ has a zero eigenvalue with multiplicity $1$. The second smallest eigenvalue $\lambda_2$ of $L_G$, also known as the \emph{Fiedler value}, reflects the connectivity of $G$. For a connected graph $G$, which is the only case of interest to us, we always have $\lambda_2>0$.

The group of all rotations about the origin in three dimensional Euclidean space 
is the \emph{Special Orthogonal Group}, denoted $\SO{3}$. 
This group is commonly represented by rotation matrices, orthogonal 
$3 \times 3$ real-valued matrices with positive determinant, i.e.,  
\begin{align}
\SO{3} \in \{ R \in \R^{3 \times 3}\ | \ R^TR=I,\ \det(R)=1 \}.
\label{embed_rotmat}
\end{align}
If we omit $\det(R)\!=\!1$, we get the \emph{Orthogonal Group},~$\OO{3}$.

We will use the convention that $\lambda_i(A)$ is the $i$:th smallest eigenvalue of the symmetric matrix $A$. The trace of matrix $A$ is denoted by $\tr{A}$ and the Kronecker product of matrices $A$ and $B$ by $A\otimes B$. The norm $\|A\|$ is the standard operator 2-norm and $\|A\|_F$ the Frobenius norm.

\section{Problem Statement}

The problem of rotation averaging is
defined as the task of determining a set of $n$ absolute 
rotations $R_1,...,R_n$ given distinct 
estimated relative rotations $\Rij$.
Available relative rotations are represented by the edge set $E$ of the camera graph $V$.
Under ideal conditions this amounts to finding the $n$ rotations compatible 
with the linear relations, 
\begin{align}
 R_i \Rij =R_j,
\label{eq_idealrot}
\end{align}
for all $(i,j)\in E$.
However, in the presence of noise, a solution to \eqref{eq_idealrot} is not 
guaranteed to exist. 
Instead, it is typically solved in a least-metric sense, 
\begin{align}
\min_{R_1,...,R_n} \sum_{(i,j)\in E} d(R_i \Rij,R_j)^p,
\end{align}
where $p\geq 1$ and $d(\cdot,\cdot)$ is a distance function.

A number of distinct choices of metrics on $\SO{3}$ exist, see Hartley {\em et al.}~\cite{hartley2013} for a comprehensive 
discussion. 
In this work we restrict ourselves to the chordal distance, the most 
commonly used metric when analyzing Lagrangian duality in rotation averaging. 
It has proven to be a convenient choice 
as it is  quadratic in its entries leading to a particularly simple derivation and form of the associated dual problem. 

The chordal distance between two rotations $R$ and $S$ is defined as their Euclidean distance in the 
embedding space,
\begin{align}
d(R,S)=\|R-S\|_F.
\end{align}
It can be shown \cite{hartley2013} that the chordal distance can also be written as $d(R,S) = 2\sqrt{2}\sin\frac{|\alpha|}{2}$, where $\alpha$ is the rotation angle of $RS^{-1}$.
With the this choice of metric, the rotation averaging problem is defined as
\begin{align}
\argmin_{R_1, ..., R_n\in \SO{3}} \sum_{(i,j) \in E } \| R_i \Rij  - R_j \|^2_F,
\label{rotavgmain}
\end{align}
which, with trace notation, can be simplified to
\begin{align}
\argmin_{R_1, ..., R_n \in \SO{3}}
-\sum_{(i,j) \in E } \tr{  R_i \Rij R_j^T },
\label{rotavgmain2}
\end{align}
which constitutes our \emph{primal problem}.

It will be convenient with a compact matrix formulation. Let
\begin{align}
\RR = \matris{ 
0 & a_{12}\Rh_{12} & \hdots & a_{1n}\Rh_{1n}\\
a_{21}\Rh_{21}& 0 & \hdots & a_{2n}\Rh_{2n} \\
\vdots & & \ddots & \vdots \\
a_{n1}\Rh_{n1}& a_{n2}\Rh_{n2} & \hdots & 0
},
\label{Rformula}
\end{align}
where $\Rh_{ij} = \Rh_{ji}^T$ and $a_{ij}$ are the elements of the adjacency matrix $A$ of the camera graph $G$
and let
\begin{equation}
R = 
\begin{bmatrix}
R_1 & R_2 & \hdots & R_n
\end{bmatrix}.
\end{equation}
We may now write the primal problem as
\begin{equation}
\begin{array}{lll}
(P) &\min & -\tr{R \RR R^T} \\
&\text{s.t.}& R \in \SO{3}^n. 
\end{array}
\end{equation}

\section{Optimality Conditions}
\subsection{Necessary Local Optimality Conditions}
We now turn to the KKT conditions of our primal problem $(P)$.
The constraint set $R \in \SO{3}^n$ consists of two types of constraints; the orthogonality constraints $R_i^T R_i = I$ and the determinant constraints $\det(R_i)=1$.

Consider relaxing the rotation averaging problem by removing the determinant constraint,
\begin{equation}
\begin{array}{lll}
(P') \quad &\min &-\tr{R \RR R^T} \\
&\text{s.t.} &R \in  \OO{3}^n. 
\end{array}
\end{equation}
The constraint $R \in \OO{3}^n$ still requires the $R_i$'s to be orthogonal. The orthogonal matrices consist of two disjoint, non-connected sets, with determinants $1$ and $-1$ respectively. Hence, any local minimizer to the problem $(P)$ also has to be a local minimizer, and therefore a KKT point, to $(P')$. We note that orthogonality can be enforced by restricting the $3\times 3$  diagonal blocks of the symmetric matrix $R^TR$ to be identity matrices.
If
\begin{equation}
\Lambda = 
\begin{bmatrix}
\Lambda_1 & 0 & 0 & \hdots \\ 
0 &\Lambda_2 &  0 & \hdots \\
 0 & 0 & \Lambda_3 & \hdots \\
 \vdots & \vdots & \vdots & \ddots
\end{bmatrix}
\label{eq:lambdadef}
\end{equation}
is a symmetric matrix then the Lagrangian can be written
\begin{equation}
\begin{array}{rl}
L(R,\Lambda) & =  -\tr{R \RR R^T} - \tr{\Lambda(I-R^TR)} \\
& = -\tr{R(\Lambda-\RR)R^T}-\tr{\Lambda}.
\end{array}
\end{equation}
Taking derivatives gives the KKT equations 
\begin{subequations}
\label{kktP}
\begin{align}
& \textit{(Stationarity) } \nonumber\\
&(\Lambda^* - \RR)R^{*^T} = 0\label{eq:KKT} \\
& \textit{(Primal feasibility) } \nonumber\\
& R^* \in \SO{3}^n.
\end{align}
\end{subequations}
Equation~\eqref{eq:KKT} states that the rows of a local minimizer $R^*$ will be eigenvectors of the matrix $\Lambda^*-\RR$ with eigenvalue zero. This allows us to compute the optimal Lagrange multiplier $\Lambda^*$ 
from a given minimizer $R^*$.
By \eqref{eq:KKT} we see that
\begin{equation} \label{eq:lambda-solution}
\Lambda^*_i R^{*T}_i = \sum_{j\neq i} a_{ij}\RR_{ij} R_j^{*T} \Longleftrightarrow
\Lambda^*_i = \sum_{j\neq i} a_{ij}\RR_{ij} R_j^{*T} R^{*}_i
\end{equation}
for $i=1,\ldots,n$.
\begin{lemma}
For a stationary point $R^*$ to the primal problem ($P$), we can compute the corresponding Lagrangian multiplier $\Lambda^*$ in closed form via (\ref{eq:lambda-solution}).
\end{lemma}

\subsection{Sufficient Global Optimality Conditions}
We begin this section by deriving the Lagrange dual of ($P$) which is a semidefinite program that we will use for optimization in later sections.
The dual problem is defined by
\begin{equation}
\quad \max_{\Lambda-\RR \succeq 0} \min_R L(R,\Lambda).
\end{equation}
Since the (unrestricted) optimum of $\min_R L(R,\Lambda)$ is either $-\tr{\Lambda}$, when $\Lambda-\RR \succeq 0$, or $-\infty$ otherwise, we get
\begin{equation}
(D) \quad \max_{\Lambda-\RR \succeq 0} -\tr{\Lambda}.
\end{equation}
It is clear (through standard duality arguments) that ($D$) gives a lower bound on ($P$). Furthermore, if $R^*$ is a stationary point with corresponding Lagrangian multiplier $\Lambda^*$ that satisfies $\Lambda^* - \RR \succeq 0$ then $\Lambda^*$ is feasible in ($D$) and by \eqref{eq:lambda-solution}, $-\tr{\Lambda^*} = - \tr{R^* \RR {R^*}^T}$, which shows that there is no duality gap between ($P$) and ($D$). Thus, the convex program ($D$) provides a way of solving the non-convex ($P$) when $\Lambda^*-\RR \succeq 0$. 


It also follows that for the stationary point $R^*$ we have $\tr{R^*\Lambda^* R^{*T}} = \tr{R^* \RR R^{*T}}$ due to $\eqref{eq:KKT}$.
We further note that if $\Lambda^*-\RR \succeq 0$ then
by definition it is true that
\begin{align}
    x^T \left(\Lambda^*-\RR \right) x \ge 0,
\end{align}
for any $3n$-vector $x$. In particular, for any $R \in \OO{3}^n$, \begin{equation}
\begin{array}{rcl}
0 & \leq & \tr{R(\Lambda^* - \RR) R^T} = \tr{\Lambda^*}-\tr{R\RR R^{T}} \\ & = & \tr{R^*\Lambda^*R^{*T}}-\tr{R\RR R^{T}},
\end{array}
\end{equation}
which shows that $-\tr{R^*\RR R^{*T}} \leq -\tr{R\RR R^{T}}$ for all $R \in \OO{3}^n$, that is, $R^*$ is the global optimum.

\begin{lemma} \label{lem:lambda-cond}
If a stationary point $R^*$ with corresponding Lagrangian multiplier $\Lambda^*$ fulfills $\Lambda^*-\RR \succeq 0$ then:
\begin{enumerate}
    \item There is no duality gap between ($P$) and ($D$).
    \item $R^*$ is a global minimum for ($P$).
\end{enumerate}
\end{lemma}
In the remainder of this paper we will study under which conditions $\Lambda^*-\RR \succeq 0$ holds and derive an efficient implementation for solving ($D$).

\section{Main Result}

In this section, we will state our main result which gives error bounds that guarentee that that strong duality holds for our primal and dual problems. From a practical point of view, the result means that it is possible to solve a convex semidefinite program and obtain the globally optimal solution to our non-convex problem, which is quite remarkable. 

\subsection{Strong Duality Theorem \label{sec:strongduality}}
Returning to our initial, primal rotation averaging problem \eqref{rotavgmain}. 
The goal is to find rotations $R_i$ and $R_j$ such that the sum of the residuals $\|R_i\Rij - R_j\|_F^2$ is minimized.
For strong duality to hold, we need to bound the residual error.
\begin{theorem}[Strong Duality]
\label{thm:strongduality1}
Let $R_i^*$, $i=1, \ldots, n$ denote a stationary point to the primal problem ($P$) for a connected camera graph $G$ with Laplacian $L_G$. Let $\alpha_{ij}$ denote the angular residuals, i.e., $\alpha_{ij} = \angle(R^*_i\Rij,R^*_j)$.
Then $R_i^*$, $i=1, \ldots, n$ will be globally optimal and strong duality will hold for ($P$) if
\begin{align} \label{eq:alpha-max0}
|\alpha_{ij}| \le \alpha_{\max} \quad \forall (i,j)\in E,
\end{align}
where \vspace{-4mm}
\begin{align} \label{eq:alpha-max}
\alpha_{\max} = 2\arcsin\left( \sqrt{\frac{1}{4}+\frac{\lambda_2(L_G)}{2d_{\max}}}-\frac{1}{2}\right),
\end{align}
and $d_{\max}$ is the maximal vertex degree.
\end{theorem}
Note that any local minimizer that fulfills this error bound will be global, and conversely there are no non-global minimizers with error residuals fulfilling \eqref{eq:alpha-max0}. 
It is clear that~\eqref{eq:alpha-max} will give a positive bound $\alpha_{\max}$ for any graph. Thus for any given problem instance, $\alpha_{\max}$ gives an explicit bound on the error residuals for which strong duality is guaranteed to hold. The strength of the bound will depend on the particular graph connectivity encapsulated by the Fiedler value $\lambda_2(L_G)$ and the maximal vertex degree $d_{\max}$. We will see that for tightly connected graphs the bound ensures strong duality under surprisingly generous noise levels. In \cite{Wilson2016} it was observed that local convexity at a point holds under similar circumstances.

\vspace{-2mm}
\paragraph{Example. } Consider a graph with $n=3$ vertices that are connected, and all degrees are equal, $d_{max}=2$. Now from the Laplacian matrix $L_G$, one easily finds that $\lambda_2=3$. This gives $\alpha_{\max}=\frac{\pi}{3}\mathrm{rad} = 60^\circ$. So, any local minimizer which has angular residuals less than $60^\circ$ is also a global solution.

\begin{figure}[t]
\begin{center}
\includegraphics[width=55mm]{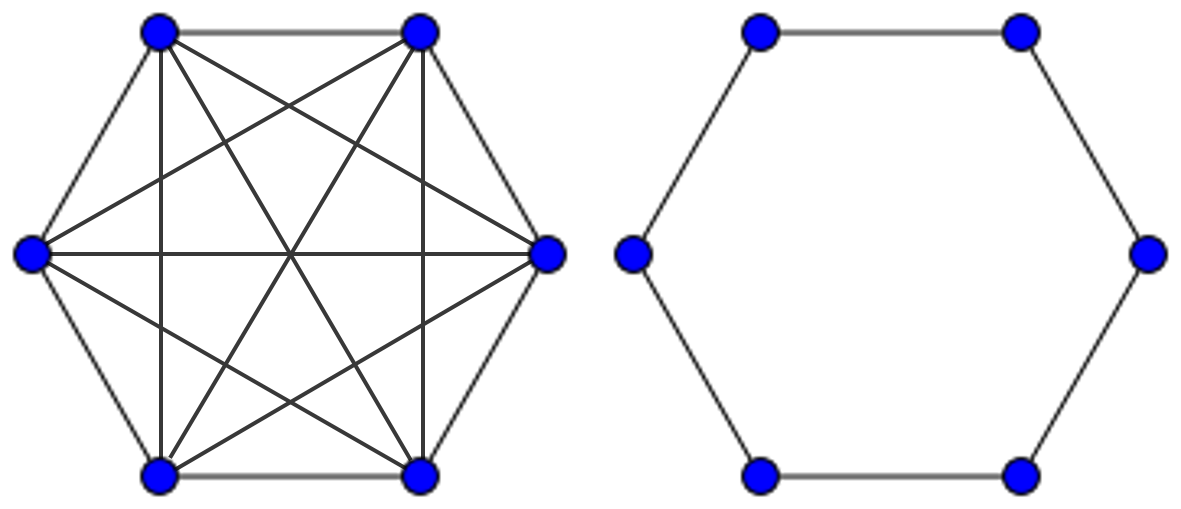}
\caption{A complete graph (left) and a cycle graph (right), both with 6 vertices.}
\label{fig:graphplot}
\end{center}
\end{figure}

\vspace{-2mm}
\paragraph{Complete graphs. }
Let us turn to a more general class of graphs, namely complete graphs with $n$ vertices, see Fig.~\ref{fig:graphplot}. As every pair of vertices is connected, it follows that $d_{\max}=n-1$. Further, it is well-known (and easy to show) that $\lambda_2(L_G)=n$, see \cite{fiedler-1973}. Again, for $n=3$, we retrieve $\alpha_{\max}=\frac{\pi}{3}\mathrm{rad}$. As $n$ becomes larger, we get a decreasing series of upper bounds which in the limit tends to $2\arcsin (\frac{\sqrt{3}-1}{2}) \approx 0.749 \mathrm{rad} = 42.9^\circ$. Hence, as long as the residual angular errors are less than $42.9^\circ$ - which is quite generous from a practical point of view - we can compute the optimal solution via a convex program. Also note that this bound holds independently of $n$.

\vspace{-1mm}
\begin{corollary} For a complete graph $G$ with $n$ vertices, the residual upper bound $\alpha_{\max} = 2\arcsin (\frac{\sqrt{3}-1}{2}) \approx 0.749 \mathrm{rad} = 42.9^\circ$ ensures global optimality and strong duality for any $n$.
\end{corollary}

\vspace{-2mm}
\paragraph{Cycle graphs. }
Now consider the other spectrum in terms of graph connectivity, namely cycle graphs. A cycle graph has a single cycle, or in other words, every vertex in the camera graph has degree two ($d_{\max}=2$) and the vertices form a closed chain (Fig.~\ref{fig:graphplot}). From the literature, we have that the Fiedler value $\lambda_2 = 2(1-\cos{\frac{2\pi}{n}})$. Inserting into \eqref{eq:alpha-max} and simplifying, we get $\alpha_{\max}= 2\arcsin\left(\sqrt{\frac{1}{4}+\sin^2(\frac{\pi}{n})}-\frac{1}{2}\right)$. Again, for $n=3$, we retrieve $\alpha_{\max}=\frac{\pi}{3}\mathrm{rad}$. For larger values of $n$, the upper bound decreases rapidly. In fact, the upper bound is quite conservative and it is possible to show a much stronger upper bound using a different analysis. In the appendix, we prove the following theorem.

\begin{theorem}
\label{thm:strongduality2}
Let $R_i^*$, $i=1, \ldots, n$ denote a stationary point to the primal problem ($P$) for a cycle graph with $n$ vertices. Let $\alpha_{ij}$ denote the angular residuals, i.e., $\alpha_{ij} = \angle(R^*_i\Rij,R^*_j)$.
Then, $R_i^*$, $i=1, \ldots, n$ will be globally optimal and strong duality will hold for ($P$) if 
 $|\alpha_{ij}| \le \frac{\pi}{n}$ for all $(i,j)\in E$.
\end{theorem}

Requiring that the angular residuals $|\alpha_{ij}|$ must be less than $\pi/n$ for the global solution may seem like a restriction, but it is actually not. To see this, note that a non-optimal solution to the rotation averaging problem can be obtained by choosing $R_1$ such that the first residual $\alpha_{12}$ is zero, and then continuing in the same fashion such that all but the last residual $\alpha_{1n}$ in the cycle is zero. In the worst case, $\alpha_{1n} = \pi$. However, this is (obviously) non-optimal. A better solution is obtained if we distribute the angular residual error evenly so that $\alpha_{ij} = \alpha = \frac{\alpha_{1n}}{n}$ (which is always possible, see Theorem~23 in \cite{enqvist-thesis-2011}). In conclusion, the angular residuals $|\alpha_{ij}|$ of the globally optimal solution for a cycle graph is always less than or equal to $\frac{\pi}{n}$, and conversely, if the angular residual is larger than $\frac{\pi}{n}$ for a local minimizer, then it does not correspond to the global solution.

In Fig.~\ref{fig:minima}, we have a real example of an orbital camera motion which is close to a cycle. It may seem hard to determine if the camera motion consists of one or more loops around the object - we give three different local minima for this example. Still, applying formula (\ref{eq:alpha-max}) for this instance gives $\alpha_{\max}=8.89^{\circ}$ which is typically sufficient in practice to ensure that the optimal solution can be obtained by solving a convex program. Before developing an actual algorithm, we shall prove our main result on strong duality.

\subsection{Proof of Theorem~\ref{thm:strongduality1}}

Recall that a sufficient condition for strong duality to hold is that $\Lambda^*-\RR \succeq 0$ (Lemma~\ref{lem:lambda-cond}). To prove Theorem~\ref{thm:strongduality1} we will show that this is true under the conditions of the theorem. 

To simplify the presentation we denote the residual rotations $\E_{ij} = R^*_i \RR_{ij} R_j^{*T}$ and define
\begin{equation} \label{eq:DR}
D_{R^*}  = \begin{bmatrix}
R_1^* & 0 & 0 & \hdots \\
 0 & R_2^* & 0 & \hdots \\
 0 & 0 & R_3^* & \hdots \\
\vdots & \vdots & \vdots & \ddots \\
\end{bmatrix}.
\end{equation}
Then $D_{R^*}(\Lambda^* - \RR)D_{R^*}^T = $
\begin{equation} \label{eq:bigeps}
\begin{bmatrix}
\sum_{j\neq 1}a_{1j}\E_{1j} & -a_{12}\E_{12} & -a_{13}\E_{13} & \hdots \\
-a_{12}\E^T_{12} & \sum_{j\neq 2}a_{2j}\E_{2j} & -a_{23}\E_{23} & \hdots \\
-a_{13}\E^T_{13} &  -a_{23}\E^T_{23} & \sum_{j\neq 3}a_{3j}\E_{3j} & \hdots \\
\vdots & \vdots & \vdots & \ddots \\
\end{bmatrix}.
\end{equation}
Note that $\sum_{j\neq i}a_{ij}\E_{ij} = \frac{1}{2} \sum_{j\neq i}a_{ij}(\E_{ij}+\E^T_{ij})$ by symmetry of $\Lambda^*$. Since $D_{R^*}$ is orthogonal, the matrix $\Lambda^* -\RR$ is positive semidefinite if and only if $D_{R^*}(\Lambda^* - \RR)D_{R^*}^T$ is.

In the noise free case we note that the residual rotations will fulfill $\E_{ij} = I$
and therefore
\begin{equation}
D_{R^*}(\Lambda^* - \RR)D_{R^*}^T = L_G\otimes I_3.
\end{equation}
In the general noise case our strategy will therefore be to bound the eigenvalues of 
$D_{R^*}(\Lambda^* - \RR)D_{R^*}^T$ by those of $L_G$ for which well-known estimates exist. Thus, we will analyze the difference and define the matrix
\begin{equation}
\Delta =  D_{R^*}(\Lambda^* - \RR)D_{R^*}^T-L_G \otimes I_3.
\label{eq:Deltadef}
\end{equation}
The following results characterize the eigenvalues of $\Delta$.
\begin{lemma} \label{lemma:delta1}
Let $\Delta_{ij}$, $i=1,...,n$, $j=1,...,n$ be the $3 \times 3$ sub-blocks of $\Delta$. 
If $\lambda$ is an eigenvalue of $\Delta$ then 
\begin{equation}
|\lambda| \leq \sum_{j=1}^n \|\Delta_{ij}\| \quad \mbox{  for some } i=1,\ldots, n.
\end{equation}
\end{lemma}
\begin{proof}
The proof is similar to that of Gerschgorin's theorem \cite{feingold1962}.
Let $\Delta x = \lambda x$, with $\|x\|=1$. Then
$\lambda x_i = \sum_j \Delta_{ij} x_j.$
Now pick $i$ such that $\|x_i\| \geq \|x_j\|$ for all $j$.
Then
\begin{equation}
|\lambda| = \left\| \lambda \frac{x_i}{\|x_i\|} \right\| =  \left\|\sum_{j=1}^n \Delta_{ij}  \frac{x_j}{\|x_i\|} \right\| \leq \sum_{j=1}^n \|\Delta_{ij}\|.
\end{equation}
\vspace{-4mm}
\end{proof}

\begin{lemma}\label{lemma:delta2} Denote $\alpha_{\max}$ the largest (absolute) residual angle of all $\E_{ij}$ and assume $0\le \alpha_{\max} \leq \frac{\pi}{2}$. Then
\begin{equation}
\|\Delta_{ii}\| \leq 2 d_i\sin^2(\frac{\alpha_{\max}}{2}) \quad \forall i=1,\ldots n,
\end{equation}
where $d_i$ is the degree of vertex $i$.
\end{lemma}
\begin{proof} It is easy to see that by applying a change of coordinates $\E_{ij}$ can be written
\begin{equation}
\E_{ij} = V_{ij} \left[\begin{matrix}
\cos(\alpha_{ij}) & -\sin(\alpha_{ij}) & 0 \\
\sin(\alpha_{ij}) & \cos(\alpha_{ij}) & 0 \\
0 & 0 & 1
\end{matrix}\right]V_{ij}^T,
\end{equation}
and therefore
\begin{equation}
\frac{1}{2}(\E_{ij}+\E_{ij}^T) = V_{ij} \left[\begin{matrix}
\cos(\alpha_{ij}) & 0 & 0 \\
0 & \cos(\alpha_{ij}) & 0 \\
0 & 0 & 1
\end{matrix}\right]V_{ij}^T.
\end{equation}
This gives 
\begin{equation}
(\cos(\alpha_{ij})-1)I \preceq \frac{1}{2}(\E_{ij}+\E_{ij}^T) - I \preceq 0,
\end{equation}
and since $\Delta_{ii} = \sum_{j \neq i} a_{ij}\left(\frac{1}{2}(\E_{ij}+\E_{ij}^T) - I\right)$ we get
\begin{equation}
d_i(\cos(\alpha_{\max})-1) I \preceq \Delta_{ii} \preceq 0.
\end{equation}
Thus $\|\Delta_{ii}\| \leq d_i (1- \cos(\alpha_{\max}))=2 d_i\sin^2(\frac{\alpha_{\max}}{2})$.
\end{proof}

\begin{lemma}\label{lemma:delta3}
If $0 \le \alpha_{max} \leq \frac{\pi}{2}$ and $i \neq j$ then 
\begin{equation}
\|\Delta_{ij}\| \le 2 a_{ij} \sin(\frac{\alpha_{\max}}{2}).
\end{equation}
\end{lemma}

\begin{proof}
To estimate the off-diagonal blocks $\|\Delta_{ij}\| = a_{ij}\|I-\E_{ij}\|$ we note that for a unit vector $v$ we have
\begin{eqnarray}
\sqrt{\|v-\E_{ij}v\|^2} &=& \sqrt{\|v\|^2-2\cos \angle(v,\E_{ij}v)+\|\E_{ij}v\|^2} \nonumber \\
&\leq& \sqrt{2(1-\cos(\alpha_{ij}))},
\end{eqnarray}
where $\angle(v,\E_{ij}v)$ is the angle between $v$ and $\E_{ij}v$.
Furthermore, we will have equality if $v$ is perpendicular to the rotation axis of $\E_{ij}$.
Therefore
\begin{equation}
\|\Delta_{ij}\| = a_{ij}\sqrt{2(1-\cos(\alpha_{ij}))} \leq 2 a_{ij}\sin(\frac{\alpha_{\max}}{2}).
\end{equation}
\end{proof}

Summarizing the results in Lemmas~\ref{lemma:delta1}-~\ref{lemma:delta3} we get that the eigenvalues $\lambda$ of $\Delta$ fulfill 
\begin{equation}
\begin{split}
|\lambda(\Delta)| \leq 2 d_i\sin^2(\frac{\alpha_{\max}}{2}) + \sum_{j\neq i} 2 a_{ij} \sin(\frac{\alpha_{\max}}{2}) \\ \leq 2d_{\max} \sin(\frac{\alpha_{\max}}{2})\left(  1+\sin(\frac{\alpha_{\max}}{2})\right),
\end{split}\label{eq:lambdamaxest}
\end{equation}
where $d_{\max}$ is the maximal vertex degree.
Note that the same bound holds for all eigenvalues of $\Delta$, in particular, the one with the largest magnitude $\lambda_{\max}(\Delta)$.

Now returning to our goal of showing that $D_{R^*}(\Lambda^* - \RR)D_{R^*}^T \succeq 0$. Let $N = \begin{bmatrix}
I & I & \hdots 
\end{bmatrix}^T$. 
The columns of $N$ will be in the nullspace of $D_{R^*}(\Lambda^* - \RR)D_{R^*}^T$.
Therefore $D_{R^*}(\Lambda^* - \RR)D_{R^*}^T$ is positive semidefinite if 
$D_{R^*}(\Lambda^* - \RR)D_{R^*}^T + \mu NN^T$ is, and hence it is enough to show that
\begin{equation}
    \lambda_1\left(D_{R^*}(\Lambda^* - \RR)D_{R^*}^T + \mu NN^T\right) \ge 0
\end{equation}
for sufficiently large $\mu$. The Laplacian $L_G$ is positive semidefinite with smallest eigenvalue $\lambda_1=0$ and corresponding eigenvector $v = \left(\begin{matrix} 1& 1& \hdots & 1\end{matrix}\right)^T$. Furthermore, as $N = v \otimes I_3$, it is clear that for 
sufficiently large $\mu$ we have $ \lambda_1(L_G\otimes I_3 + \mu NN^T) = \lambda_1(L_G+\mu vv^T) = \lambda_2(L_G)$.
Since
\begin{equation}
D_{R^*}(\Lambda^* - \RR)D_{R^*}^T + \mu NN^T = L_G\otimes I_3 + \mu NN^T+\Delta,    
\end{equation}
we therefore get
\begin{equation}
\begin{split}
\lambda_1(D_{R^*}(\Lambda^* - \RR)D_{R^*}^T + \mu NN^T) \ge \lambda_2(L_G) - |\lambda_{\max}(\Delta)|.
\end{split}
\end{equation}
If the right-hand side is positive, then so is the left-hand side.
Using \eqref{eq:lambdamaxest} for $\lambda_{\max}(\Delta)$ yields the following result.
\begin{lemma}
The matrix $\Lambda^* - \RR$ is positive semidefinite if
\begin{equation}
\lambda_2(L_G) - 2d_{\max} \sin(\frac{\alpha_{\max}}{2})\left( 1+\sin(\frac{\alpha_{\max}}{2})\right) \geq 0.
\end{equation}
\end{lemma}
By completing squares, one obtains the equivalent condition
\begin{equation}
\left(\sin(\frac{\alpha_{\max}}{2})+\frac{1}{2}\right)^2 \le
\frac{\lambda_2(L_G)}{2d_{\max}}+\frac{1}{4},
\end{equation}
which shows Theorem~\ref{thm:strongduality1}.

\section{Solving the Rotation Averaging Problem \label{sec:algorithm}}

The dual problem ($D$) is a convex semidefinite program, 
and although it is theoretically sound and 
provably solvable in polynomial time by interior point methods 
\cite{boyd-vandenberghe-book-2004}, in practice such problems quickly become intractable as the dimension of the entering variables grow. 

In this section we present a first-order method for solving semidefinite programs 
with constant block diagonals. Our approach solves the dual of ($D$) and consists of two simple matrix operations only, matrix multiplication and 
square roots of $3\times 3$ symmetric matrices, the latter which can be solved in closed form. 
Consequently, these two operations permit a simple and efficient implementation
without the need for dedicated numerical libraries. 

The dual of ($D$) is given by
\begin{equation}
\min_{Y \succeq 0} \max_{\Lambda} -\tr{\Lambda} + \tr{Y(\Lambda-\RR)}.
\end{equation}
Let the matrix $Y$ be partitioned as follows, 
\begin{align}
Y=\matris{
Y_{11} & Y_{12} & \hdots & Y_{1n} \\ 
Y_{12}^T & Y_{22} & \hdots & Y_{2n} \\ 
\vdots & \vdots & \ddots & \vdots \\ 
Y_{1n}^T &  \hdots  & \hdots & Y_{nn} \\ 
}
\label{partitioning}
\end{align}
where each block $Y_{ij}\in \R^{3\times 3}$ for $i,j=1,\ldots,n$.
Since $\Lambda$ is block-diagonal \eqref{eq:lambdadef} it is clear that the inner maximization is unbounded when $Y_{ii} - I_{3\times 3}\neq 0$ and zero otherwise. We therefore get 
\begin{equation}
\label{rotavgDDr}
\begin{array}{lll}
(DD) & \min\limits_Y & 
-\tr{\RR Y} \\ 
& \st
& Y_{ii} =I_3,  \quad i=1,...,n,  \\
& & Y \succeq 0. 
\end{array}
\end{equation}
Since $Y \succeq 0$ it is clear that
$$
-\tr{\Lambda} + \tr{Y(\Lambda-R^*)} \geq -\tr{\Lambda},
$$
for all $\Lambda$ of the form \eqref{eq:lambdadef}. Therefore $(DD) \geq (D)$ and assuming strong duality 
holds $(D) = (P)$. Furthermore if $R^*$ is the global optimum of $(P)$ then $Y = {R^*}^T R^*$ is feasible in 
\eqref{rotavgDDr} which shows that $(DD) = (P)$.

Thus, when strong duality holds, recovering a primal solution to $(P)$ is then achieved by simply  
reading off the first three rows of $Y^*$ and choosing their signs to ensure positive determinants 
of the resulting rotation matrices, see supplementary material for further details. 


\subsection{Block Coordinate Descent}

In this section we present a block coordinate descent method for solving semidefinite programs with 
block diagonal constraints on the form \eqref{rotavgDDr}. This method is a generalization of the 
row-by-row algorithms derived in \cite{wen2009row}. 

Consider the following semidefinite program, 
\begin{equation}
\label{rotavgblock}
\begin{array}{ll}
\min\limits_{\mathclap{S \in \R^{3n\times 3}}} & \hspace{2mm} \tr{\AA^TS} \\
\st 
& \hspace{2mm} \matris{
I  & S^T   \\ 
S & B 
} \succeq 0.  
\end{array}
\end{equation}
This is a subproblem that arises when attempting to solve ($DD$) in \eqref{rotavgDDr} using a 
block coordinate descent approach, i.e., by fixing all but one row and column of blocks 
in \eqref{partitioning} and reordering as necessary. 
It turns out that this subproblem has a particularly simple, closed form solution, established by the following lemma.
\begin{lemma}
\label{lemmaavgblocksol}
Let $B$ be a positive semidefinite matrix. Then, the solution to \eqref{rotavgblock} 
is given by, 
\begin{align}
S^*= -B \AA \left[ \Big(\AA^T B \AA \Big)^{\frac{1}{2}} \right]^\dagger.  
\label{rotavgblocksol}
\end{align}
\end{lemma}
Here $^\dagger$ denotes the Moore--Penrose pseudoinverse. 
\begin{proof}
See supplementary material. 
\end{proof}

\begin{algorithm}[hbt]
\vskip 2pt
\caption{ A block coordinate descent algorithm for the semidefinite relaxation ($DD$) in \eqref{rotavgDDr}. \label{alg1} }
{\normalsize
\begin{algorithmic}[]
  \STATE \begin{varwidth}[t]{\linewidth} {\bfseries input:} 
	  \hskip\algorithmicindent $\RR$, $Y^{(0)} \succeq 0$,\ \  $t=0.$
      \end{varwidth}
	\REPEAT
	\STATE   $\boldsymbol{\cdot}$ Select an integer $k \in [1,\ldots,n]$,  \\
	\STATE   $\boldsymbol{\cdot}$ $B_k$: the result of eliminating the k\textsuperscript{th} row 
	and column from $Y^t$.
	\STATE   $\boldsymbol{\cdot}$ $\AA_k$: the result of eliminating the k\textsuperscript{th} column
	and all but the k\textsuperscript{th} row from $\RR$.
    \STATE   $\boldsymbol{\cdot}$ 
    $S_k^*= -B_k \AA_k \big[ \big(\AA_k^T B_k \AA_k \big)^{\frac{1}{2}} \big]^\dagger$ as in \eqref{rotavgblocksol}.
    \STATE   $\boldsymbol{\cdot}$  $Y^{t} = \matris{ I  & S_k^{*T} &  \\ S_k^{*} & B_k }$,  {(succeeded by the appropriate reordering).}
	\STATE $\boldsymbol{\cdot}$  $t = t+1$
	\UNTIL{convergence}
   \end{algorithmic} }
\vskip 1pt
\end{algorithm}


\vspace{-5mm}
\section{Experimental Results}
\begin{table*}[htb]
\centering
\ra{1.1}
\begin{tabular}{@{}rrrrcrrrcrrr@{}}
\toprule
& \multicolumn{2}{c}{LM \cite{nocedal1999}} & \phantom{abc}& \multicolumn{2}{c}{Alg. \ref{alg1}} &
\phantom{abc} & \multicolumn{2}{c}{SeDuMi \cite{sedumi} }\\
\cmidrule{2-3} \cmidrule{5-6} \cmidrule{8-9}
$n$ \hspace{8mm} $\sigma$ [rad]\ \ & 
$avg. error$ (\%)&  $time [s]$&& 
$avg. error$ &  $time [s]$&& 
$avg. error$ & $time [s]$\\ \midrule
%
%
%
$20$ \hspace{14mm} $0.2 $ &  1.49 (0.48) & 0.012  &&     9.34e-10 & 0.028 &&     4.30e-09  & 0.501 \\ 
$$ \hspace{14mm} $0.5 $ &  0.56 (0.73) & 0.008  &&     3.94e-08 & 0.023 &&     3.72e-09  & 0.553 \\ 
$50$ \hspace{14mm} $0.2 $ &  0.55 (0.50) & 0.026  &&     1.3e-09 & 0.17 &&     6.85e-09  & 5.91 \\ 
$$ \hspace{14mm} $0.5 $ &  0.17 (0.58) & 0.017  &&     1.83e-07 & 0.33 &&     2.00e-09  & 6.32 \\ 
$100$ \hspace{14mm} $0.2 $ &  0.15 (0.55) & 0.042  &&     1.46e-07 & 8.89 &&     5.31e-09  & 47.0 \\ 
$$ \hspace{14mm} $0.5 $ &  0.15 (0.45) & 0.039  &&     6.64e-08 & 7.97 &&     7.41e-10  & 49.51 \\ 
$200$ \hspace{14mm} $0.2 $ &  0.099 (0.40) & 0.082  &&     4.02e-08 & 17.01 &&     4.15e-10  & 419.04 \\  
$$ \hspace{14mm} $0.5 $ &  0.031 (0.33) & 0.071  &&     6.79e-08 & 29.4 &&     6.91e-10  & 391.23 \\ 
\bottomrule
\end{tabular}
\caption{Comparison of running times and resulting errors on synthetic data. 
Here the errors are given with respect to the lowest feasible objective function value 
found. The fraction of the times the global optima was reached by the LM algorithm is indicated
along side the average error. \\
\label{table1}
}
\end{table*}
\begin{figure*}
\def\ww{27mm}
    \centering
    \includegraphics[width=\ww]{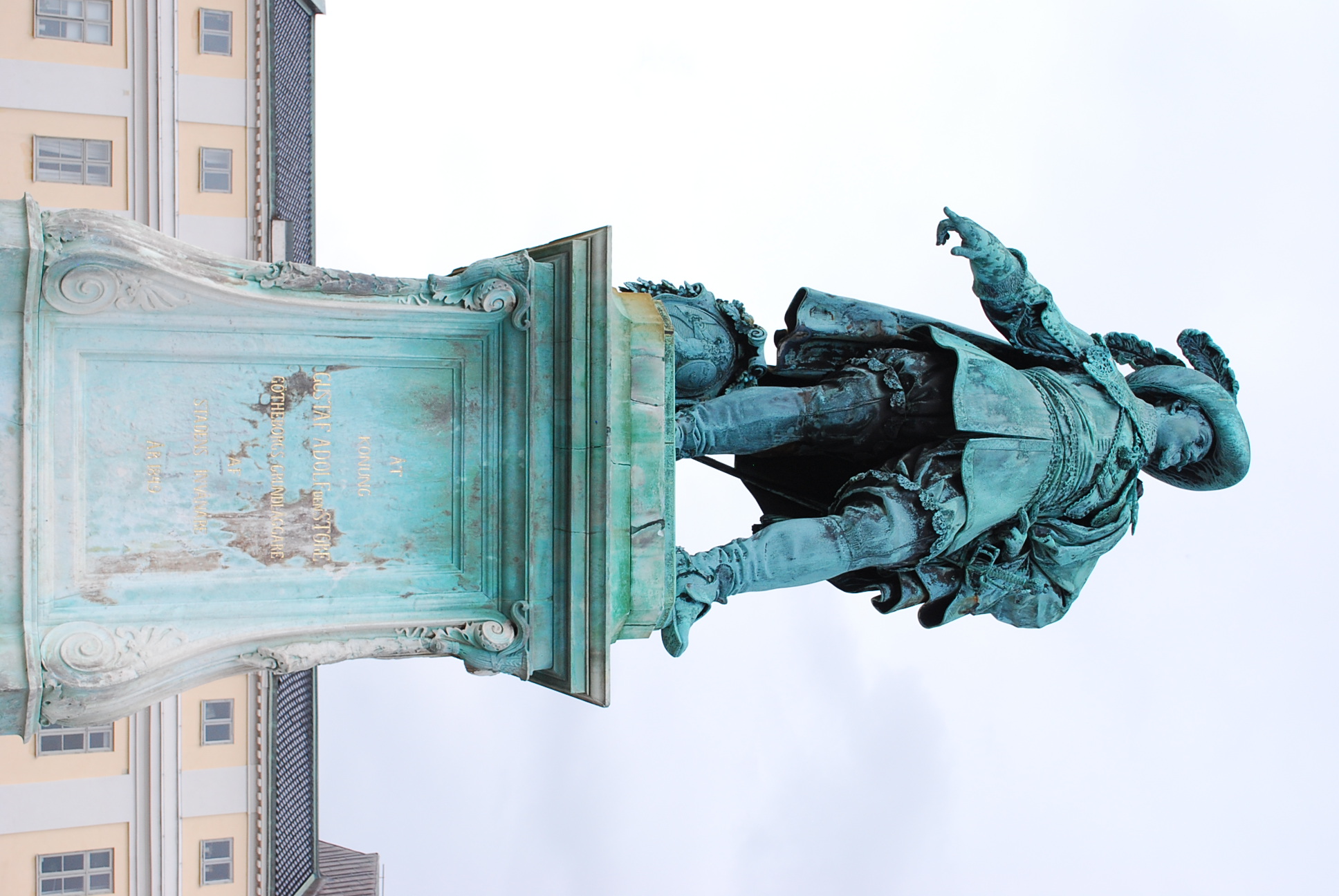} \hspace{.2cm}
    \includegraphics[width=\ww]{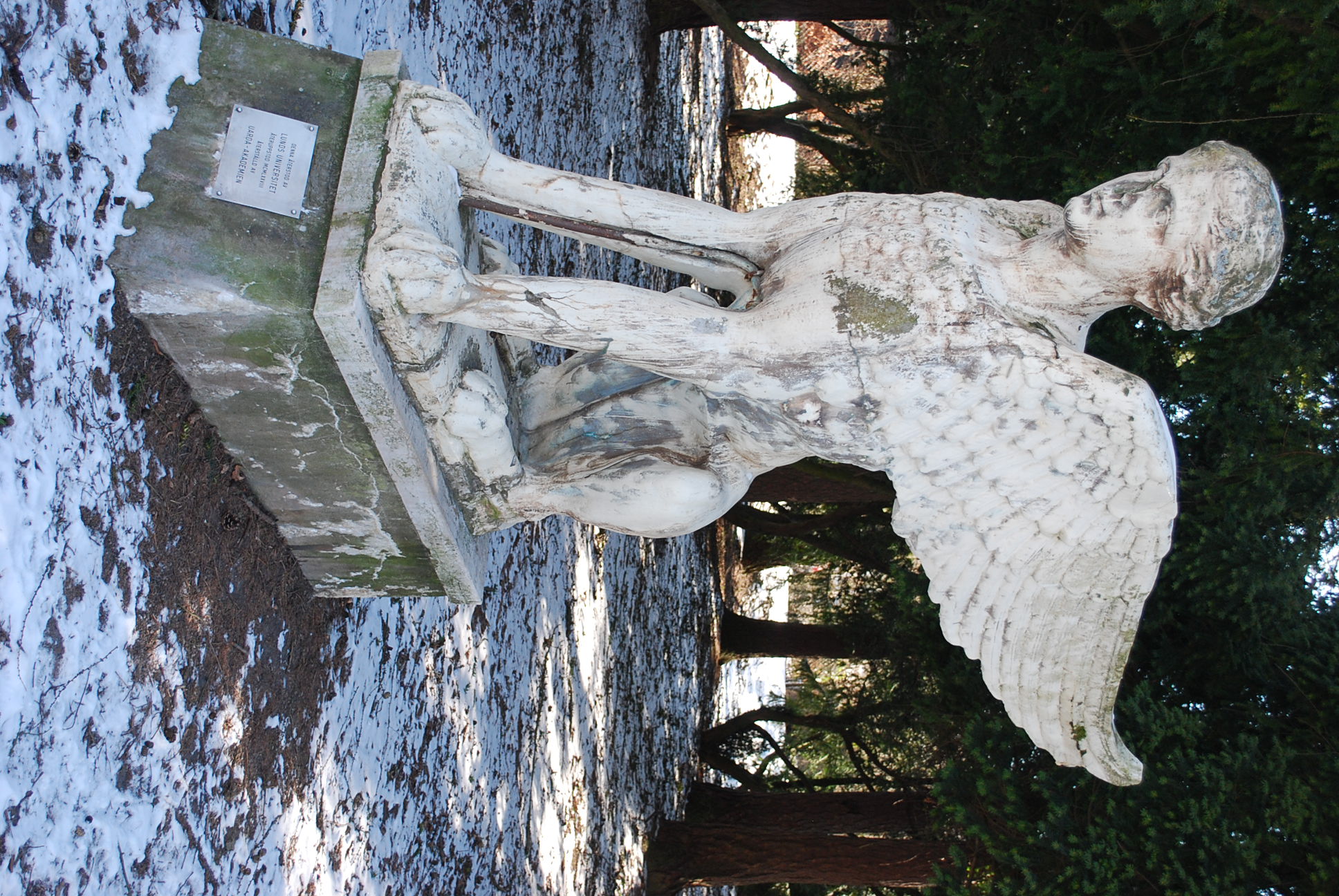} \hspace{.2cm}
    \includegraphics[width=\ww]{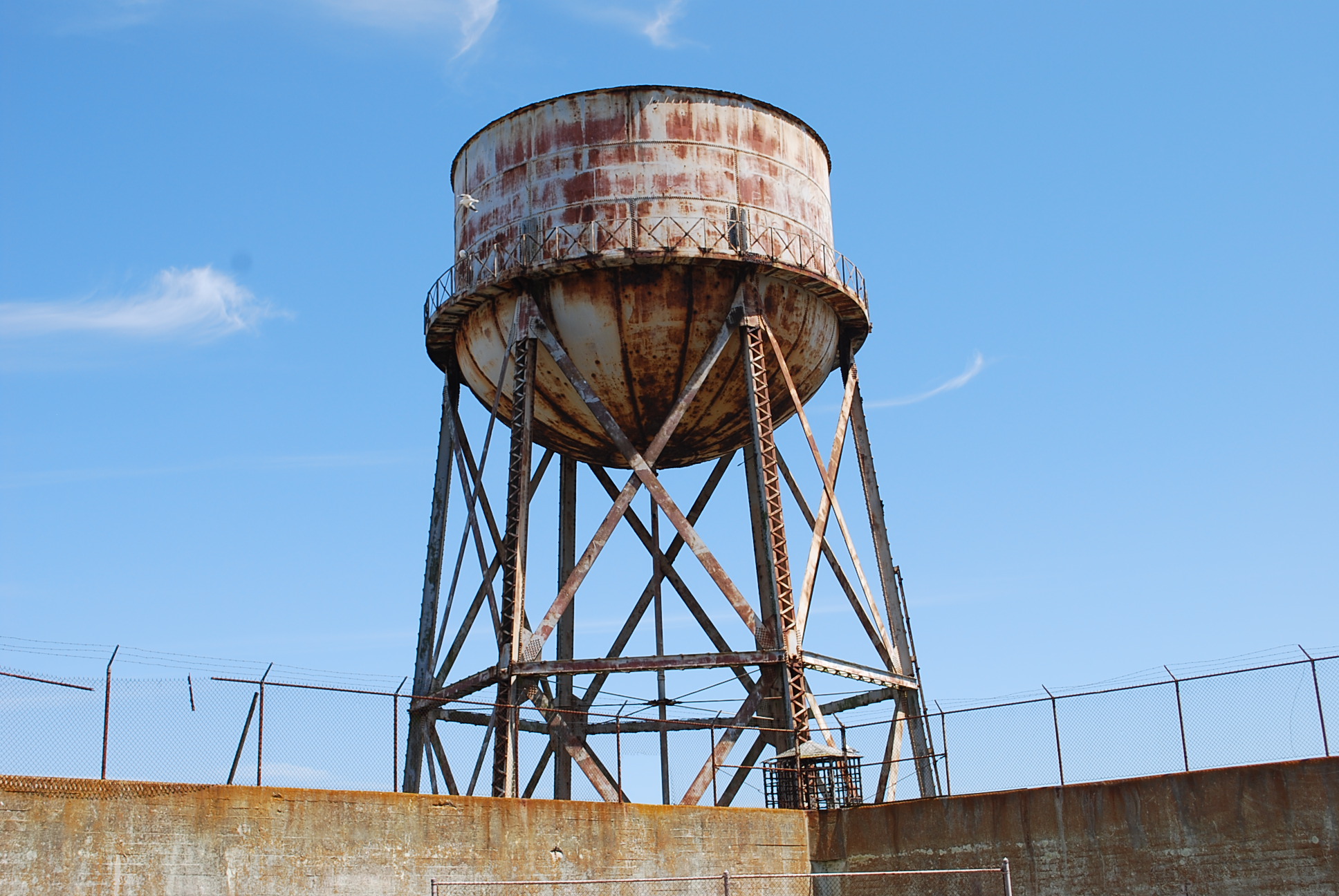} \hspace{.2cm}
    \includegraphics[width=\ww]{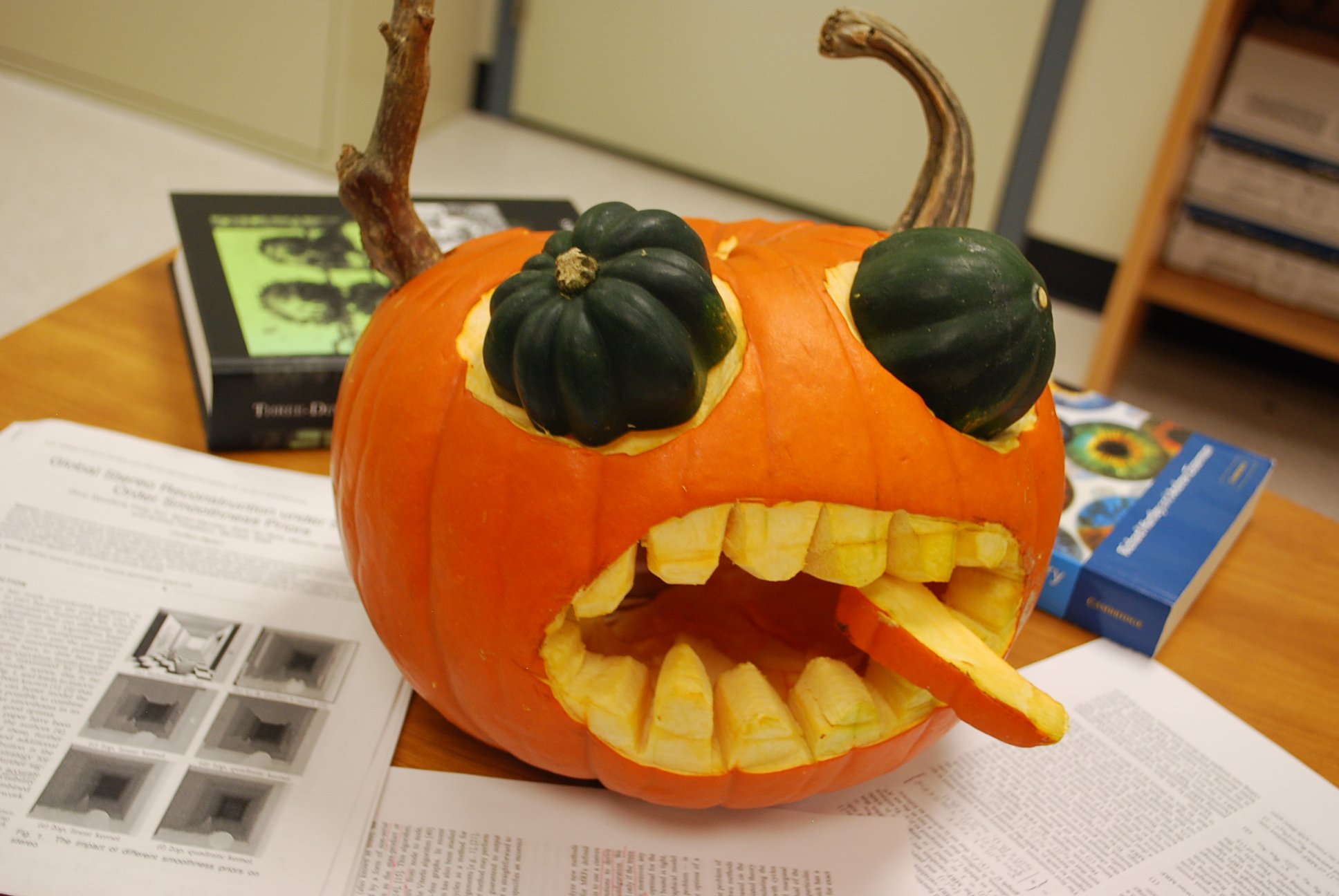} \hspace{.2cm}
    \includegraphics[width=\ww]{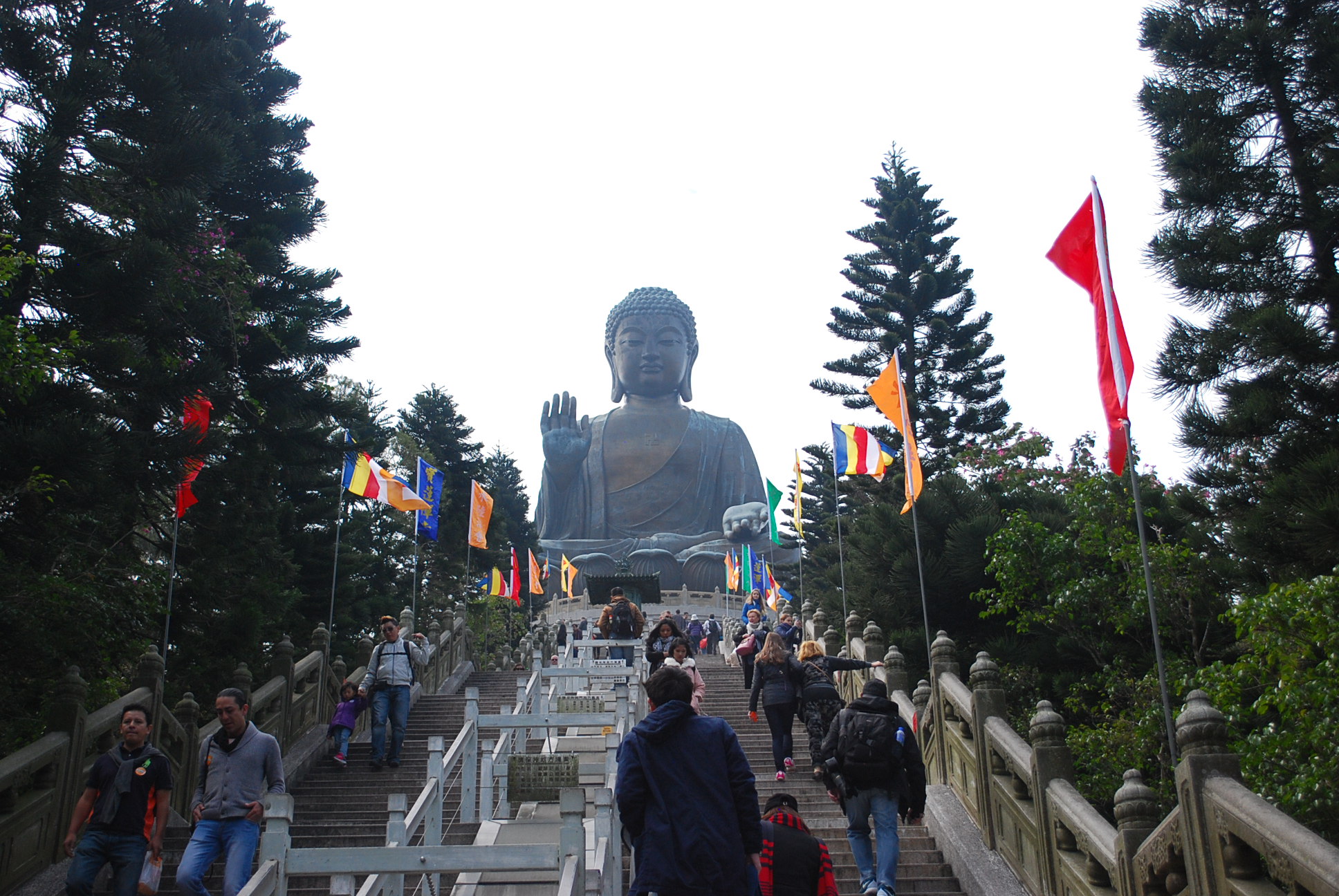} \\
    \includegraphics[width=\ww]{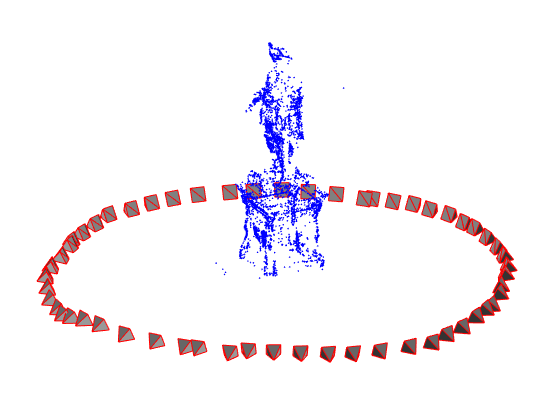} \hspace{.2cm}
    \includegraphics[width=\ww]{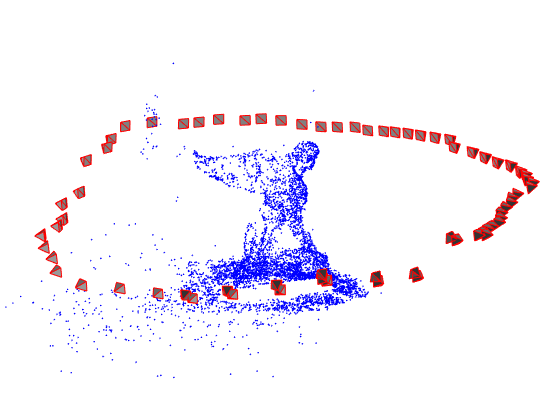} \hspace{.2cm}
    \includegraphics[width=\ww]{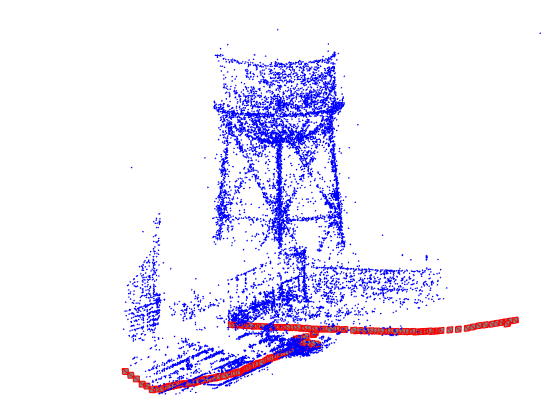} \hspace{.2cm}
    \includegraphics[width=\ww]{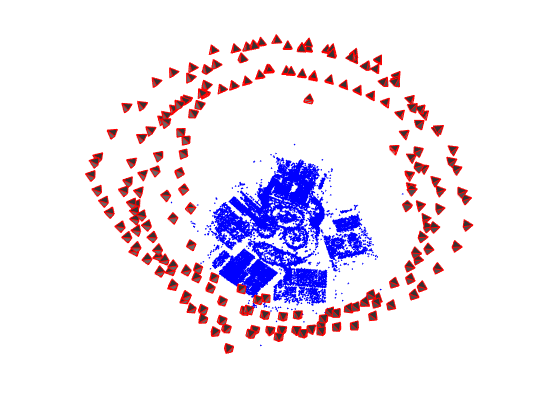} \hspace{.2cm}
    \includegraphics[width=\ww]{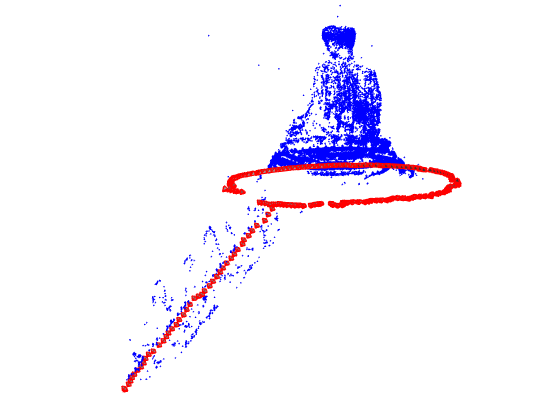} \\
    \vspace{-2mm}
    \caption{Images and reconstructions of the datasets in Table~\ref{table2}.}
    \label{fig:datasets}
\end{figure*}

In this section we present an experimental study aimed at characterizing the 
performance and computational efficiency of the proposed algorithm compared to existing standard numerical solvers.

\vspace{-2mm}
\paragraph{Synthetic data. }
In our first set of experiments we compared the computational efficiency 
of the Levenberg-Marquardt (LM) algorithm \cite{nocedal1999}, a standard nonlinear 
optimization method, 
Algorithm~\ref{alg1} and that of SeDuMi~\cite{sedumi}, a publicly available 
software package for conic optimization.

We constructed a large number of synthetic problem instances of increasing  size, perturbed by varying levels of noise. 
Each absolute rotation was obtained by rotation about the z-axis by $2\pi/n$ rad and by construction, forming a cycle graph.
The relative rotations were perturbed by noise in the form of a random rotation 
about an axis sampled from a uniform distribution on the unit sphere with angles normally distributed with mean $0$ and variance $\sigma$. 
The absolute rotations were initialized (if required) in a similar fashion but with the 
angles uniformly distributed over $[0,2\pi]$ rad. 

The results, averaged over $50$ runs, can be seen in Table~\ref{table1}. 
As expected, the LM algorithm significantly outperforms our algorithm as well as 
SeDuMi, but it only manages to obtain the global optima in about $30-70\%$ of the time. 
As predicted by Theorem~\ref{thm:strongduality2} and the discussion in Section~\ref{sec:strongduality} on cycle graphs, both Algorithm~\ref{alg1} and SeDuMi produce globally optimal solutions at every single problem instance, independent of the noise level and independent on the number of cameras.
From this table we also observe that Algorithm~\ref{alg1} does appear to outperform 
SeDuMi quite significantly with respect to computational efficiency. 

\begin{table}[htb]
\centering
\ra{1.1}
\begin{tabular}{@{}lrrrcrrrcrrr@{}}
\toprule
&& \multicolumn{2}{c}{$time [s]$} \\
\cmidrule{3-4} 
$Dataset$ & $n$  & Alg.~\ref{alg1} &  SeDuMi & 
$  |\alpha_{ij}|$ & $\alpha_{\max}$ \\
\midrule
Gustavus   &  57  &  3.25  &  8.28  & $6.33^\circ$ & $8.89^\circ$	\\ 
Sphinx &  70  &  3.87  &  14.40  & $6.14^\circ$	& $12.13^\circ$ \\
Alcatraz &  133 &  12.73 & 117.19 & $7.68^\circ$ & $43.15^\circ$\\ 
Pumpkin &  209  &  9.23  &  688.65  & $8.63^\circ$	& $3.59^\circ$ \\ 
Buddha &  322  &  16.71  &  1765.72  & $7.29^\circ$	& $14.01^\circ$\\ 
\bottomrule
\end{tabular}
\caption{The average run time and largest resulting angular residual 
($|\alpha_{ij}|$) and bound ($\alpha_{\max}$)  
on five different real-world datasets. \label{table2}
}
\end{table}

\vspace{-2mm}
\paragraph{Real-world data.}
In our second set of experiments we compared the computational efficiency on a number of 
publicly available real-world datasets \cite{enqvist2011non}. The results, again averaged over 
$50$ runs,  are presented in Table~\ref{table2}. Here, as in the previous experiment, both methods correctly produce the global optima 
at each instance. Algorithm \ref{alg1} again significantly outperforms SeDuMi 
in computational cost, providing further evidence of the efficiency of the proposed algorithm. 
It can further be seen that Theorem~\ref{thm:strongduality1} provides bounds sufficiently large to guarantee strong 
duality, and hence global optimality, in all the real-world instances except for one, the 
\textit{Pumpkin} dataset. 
Although strong duality does indeed hold in this case, the resulting certificate is less than the largest 
angular residual obtained. 
The camera graph 
is comprised both of densely as well as sparsely connected cameras, 
resulting in a large value of $d_{\max}$ in combination with a small value of $d_{\min}$ (minimum degree). 
Since $\lambda_2 \leq  d_{\min}$ a limited bound on $\alpha_{\max}$ follows directly from 
\eqref{eq:alpha-max}. This instance serves as a representative example of when the bounds of Theorem~\ref{thm:strongduality1}, although still valid and strictly positive, become too conservative in practice.


\section{Conclusions} 
In this paper we have presented a theoretical analysis of Lagrangian 
duality in rotation averaging based on spectral graph theory. 
Our main result states that for this class of problems strong duality will provably 
hold between the primal and dual formulations if the noise levels are sufficiently restricted. 
In many cases the noise levels required for strong duality not to hold can be shown to be quite severe. 
To the best of our knowledge, this is the first time such practically useful sufficient conditions for 
strong duality have been established for optimization over multiple rotations.  

A scalable first-order algorithm, a generalization of 
coordinate descent methods for semidefinite cone programming, was also presented. Our empirical validation 
demonstrates the potential of this proposed algorithm, significantly outperforming existing general purpose 
numerical solvers.

\bibliographystyle{ieee}
{\small
\bibliography{cvpr2017dualrotation}
}

\input{cvpr2017laplacian_suppmaterial.tex}

\end{document}

%% file: Macros.tex
\def \R{{\mathbb{R}}}

\def \E{{\mathcal{E}}}

\def \0{{\vec{0}}}

\def \alg1 {algorithm \ref{alg1} }
\def \Rij {\tilde{R}_{ij}}
\def \Rh {\tilde{R}}
\def \RR {\tilde{R}}

\newcommand{\trace}[2]{<#1,#2>}
\def \st{{\textnormal{s.t.}}}

\newcommand{\SO}[1]{ \textnormal{SO}(#1) }

\newcommand{\comment}[1]{}

\newtheorem{theorem}{Theorem}[section]
\newtheorem{lemma}{Lemma}[section]
\newtheorem{corollary}{Corollary}[section]

\newcommand{\matris}[1]{ \left[ \begin{smallmatrix} #1 \end{smallmatrix} \right]}

\newcommand{\Arraj}[1]{  \begin{array}{ccccccccccc} #1 \end{array} }

\newcommand{\argmin}{\operatornamewithlimits{arg\ min}}



%% file: cvpr2017laplacian_suppmaterial.tex
\newpage

\twocolumn[ 
\begin{center}
{\Large \bf Rotation Averaging and Strong Duality - Supplementary Material\par}
\vspace*{24pt}
      {
      \large
      \lineskip .5em
      \begin{tabular}[t]{c} 
         \ifcvprfinal Anders Eriksson, Carl Olsson, Fredrik Kahl and Tat-Jun Chin
\else Anonymous CVPR submission\\
         \vspace*{1pt}\\
        Paper ID \cvprPaperID \fi
      \end{tabular}
      \par
      }
      \vskip .5em
      \vspace*{12pt}
\end{center}
]


\subsection*{Proof of Theorem~\ref{thm:strongduality2}}
\setcounter{section}{4}
\setcounter{theorem}{1}

\begin{theorem}
\label{thm:strongduality2b}
Let $R_i^*$, $i=1, \ldots, n$ denote a stationary point to the primal problem ($P$) for a cycle graph with $n$ vertices. Let $\alpha_{ij}$ denote the angular residuals, i.e., $\alpha_{ij} = \angle(R^*_i\Rij,R^*_j)$.
Then, $R_i^*$, $i=1, \ldots, n$ will be globally optimal and strong duality will hold for ($P$) if
 $$
 |\alpha_{ij}| \le \frac{\pi}{n} \quad \forall (i,j)\in E.$$
\end{theorem}

\begin{proof}
A sufficient condition for strong duality to hold is that $\Lambda^*-\RR \succeq 0$ (Lemma~\ref{lem:lambda-cond}), which is equivalent to $D_{R^*}(\Lambda^* - \RR)D_{R^*}^T \succeq 0$ with the same notation and argument as in (\ref{eq:DR}) and (\ref{eq:bigeps}). For a cycle graph, we get
$D_{R^*}(\Lambda^* - \RR)D_{R^*}^T = $
\begin{equation} \label{eq:bigeps2}
\begin{bmatrix}
\E_{12}+\E_{1n} & -\E_{12} & & & -\E_{1n} \\
-\E^T_{12} & \E^T_{12} + \E_{23} & -\E_{23}  \\
   &  -\E^T_{23} & \ddots & \ddots \\
 & & \ddots & \ddots &  \\ 
 -\E^T_{1n} 
\end{bmatrix}.
\end{equation}
As this matrix is symmetric, it implies for the first diagonal block that $\E_{12}-\E^T_{12}=\E^T_{1n}-\E_{1n}$. As all $\E_{ij} \in \SO{3}$, it follows that $\E_{12}=\E^T_{1n}=\E$ for some rotation $\E \in \SO{3}$. Similarly, for the second diagonal block $\E_{12} = \E^T_{23} = \E$ and by induction, the matrix $D_{R^*}(\Lambda^* - \RR)D_{R^*}^T$ has the following tridiagonal (Laplacian-like) structure
\begin{equation} \label{eq:cycle-laplace}
\begin{bmatrix}
\E \!+\!\E^T & -\E      &  &  & -\E^T \\
-\E^T   & \E \!+\!\E^T & -\E &  &  &  \\
        &  -\E^T & \ddots & \ddots  \\
  &  &  \ddots & \ddots & -\E \\
 -\E &  &  & -\E^T & \E\!+\!\E^T \\
\end{bmatrix}.
\end{equation}
Note that this means that the total error is equally distributed in an optimal solution among all the residuals, in particular, $\alpha_{ij}=\alpha$ for all $(i,j)\in E$, where $\alpha$ is the residual rotation angle of $\E$.


Let $v$ denote the rotation axis of $\E$ and let $u$ and $w$ be an orthogonal base which is orthogonal to $v$. Then, define the two vectors $v_{\pm} = (\, v_{\pm,1} \quad v_{\pm,2} \enspace \ldots \enspace v_{\pm,n} \,)^T$, where $v_{\pm,i} = \cos(\frac{2\pi i}{n}) u \pm \sin( \frac{2\pi i}{n}) w$ for $i=1,\ldots,n$. Now it is straight-forward
to check that $v_{\pm}$ are eigenvectors to (\ref{eq:cycle-laplace}) with eigenvalues $4\sin(\frac{\pi}{n} \pm \alpha)\sin(\frac{\pi}{n})$. The sign of the smallest of these two eigenvalues determines the positive definiteness of the matrix in (\ref{eq:cycle-laplace}). In other words,
we have shown that if $|\alpha| \le \frac{\pi}{n}$ then $D_{R^*}(\Lambda^* - \RR)D_{R^*}^T \succeq 0$.
\end{proof}


\subsection*{Proof of Lemma~\ref{lemmaavgblocksol}}
\setcounter{section}{5}
\setcounter{theorem}{0}

\begin{lemma}
\label{lemmaavgblocksol2}
Let $B$ be a positive semidefinite matrix. Then, the solution to \eqref{rotavgblock} 
is given by, 
\begin{align}
S^*= -B \AA \left[ \Big(\AA^T B \AA \Big)^{\frac{1}{2}} \right]^\dagger.  
\label{rotavgblocksol2}
\end{align}
\end{lemma}
\begin{proof}
From the Schur complement, we have that the $2\times 2$ block matrix in 
\eqref{rotavgblock} is positive semidefinite if and only if 
\begin{align}
I - S^T B^{\dagger}S \succeq 0,\\
(I-B B^\dagger)S=0.
\label{sdpequal}
\end{align}
Hence the problem \eqref{rotavgblock} is equivalent to 
\begin{subequations}
\label{rotavgblock2}
\begin{align}
\min_{\mathclap{S \in \R^{3n\times 3}}} &  \hspace{5mm} \trace{\AA}{S} \\
 \st &\ \ \hspace{5mm}  I-S^T B^{\dagger} S \succeq 0, \\
&\ \ \hspace{5mm}  (I-B B^\dagger)S=0. 
\end{align}
\end{subequations}
The KKT conditions for \eqref{rotavgblock2}, with Lagrangian multipliers $\Gamma$ and
$\Upsilon$ , become
\begin{align}
\AA + 2 B^{\dagger} S \Gamma  +  (I-B B^\dagger)\Upsilon =0, \label{blockkkta}\\ 
I-S^T B^{\dagger} S \succeq 0,\label{blockkktb}\\
(I-B B^\dagger)S=0,  \label{blockkktb2}\\
\Gamma \succeq 0,\label{blockkktc}\\   
(I-S^T B^{\dagger} S) \Gamma =0. \label{blockkktd}
\end{align}
Rewrite \eqref{blockkkta} and \eqref{blockkktd} as 
\begin{align}
B^{\dagger} S \Gamma &=-\frac{1}{2}\AA  -\frac{1}{2}  (I-B B^\dagger)\Upsilon, \label{blockkkta2}\\ 
\Gamma^T\Gamma  &= \Gamma^T S^T B^{\dagger} S \Gamma. \label{blockkktd2}
\end{align}
Since the pseudoinverse fulfills $B^\dagger B B^\dagger = B^\dagger$, combining \eqref{blockkkta2} and \eqref{blockkktd2} we obtain 
\begin{align}
&\Gamma^2  = \Gamma^T S^T B^{\dagger} B B^{\dagger} S \Gamma = \\
&=\frac{1}{4}
\left( \AA  + (I-B B^\dagger)\Upsilon \right)^T B \left( \AA  + (I-B B^\dagger)\Upsilon \right) = \\
&=\frac{1}{4} \AA^T B \AA.  \label{blockkkt3}
\end{align}
Here the last equality follows since $B(I-B B^\dagger)=0$. 
This gives
\begin{align}
\Gamma  =  \frac{1}{2} \Big( \AA^T B \AA \Big)^{\frac{1}{2}} \label{blockkkt4}. 
\end{align}
Inserting \eqref{blockkkt4} in \eqref{blockkkta2} 
\begin{align}
B^{\dagger} S \Big( \AA^T B \AA \Big)^{\frac{1}{2}} &=-\AA  -  (I-B B^\dagger)\Upsilon, \\
\end{align}
multiplying with $B$ form the left on both sides and using \eqref{blockkktb2}, $B B^\dagger S = S $, 
we arrive at 
\begin{align}
S  \Big( \AA^T B \AA \Big)^{\frac{1}{2}} &=-B\AA, 
\end{align}
and consequently  
\begin{align} 
S  &=-B \AA \left[ \Big( \AA^T B \AA \Big)^{\frac{1}{2}} \right]^\dagger.
\label{blockkkt5}
\end{align}
Finally, since 
\begin{align}
&\Gamma =  \frac{1}{2} \Big( \AA^T B \AA \Big)^{\frac{1}{2}}  \succeq 0, \\
& I-S^T B^\dagger S =  \nonumber \\
&  \hspace{5mm} =I- \left[ \Big( \AA^T B \AA \Big)^{\frac{1}{2}} \right]^\dagger \AA^T B \AA 
\left[ \Big( \AA^T B \AA \Big)^{\frac{1}{2}}\right]^\dagger \succeq 0, 
\end{align}
the conditions \eqref{blockkktb} and  \eqref{blockkktc} are satisfied 
then 
\eqref{rotavgblocksol2} must  be a feasible and optimal solution 
to \eqref{rotavgblock2} and consequently also to \eqref{rotavgblock}. 
\end{proof}
